\documentclass{article}

 \usepackage[preprint]{neurips_2026}

\usepackage[utf8]{inputenc} 
\usepackage[T1]{fontenc}    
\usepackage{hyperref}       
\usepackage{url}            
\usepackage{booktabs}       
\usepackage{amsfonts}       
\usepackage{nicefrac}       
\usepackage{microtype}      
\usepackage{xcolor}         

\usepackage{amsmath}        
\usepackage{amsthm}
\usepackage{graphicx}
\usepackage{mathtools}

\newtheorem{theorem}{Theorem}

\title{Learning in Deep Networks under Dale's Constraint }

\author{%
  Roy Abel \\
  Weizmann Institute of Science \\
  \texttt{roy.abel@weizmann.ac.il} \\
  \And
  Shimon Ullman \\
  Weizmann Institute of Science \\
  \texttt{shimon.ullman@weizmann.ac.il} \\
}

\begin{document}

\maketitle

\begin{abstract}
  Biologically plausible learning models aim to explain how neural circuits can implement effective learning under the constraints of real neurons. Although significant progress has been made, a major remaining challenge is that existing models often allow neurons or synapses to represent mixed-sign values, both positive and negative, in violation of a basic aspect of cortical circuitry -- Dale's constraint: biological neurons are either excitatory or inhibitory, but not both, and synapses cannot change sign. In this work, we address this discrepancy by introducing a biologically motivated neural architecture in which both neural activations and learning signals are represented by non-negative activity, and synapses have fixed sign, while still supporting backpropagation-like learning. Our approach uses two complementary interacting non-negative channels to represent positive and negative contributions, inspired by evidence of on-off representations in the brain. These channels are implemented through a simple neural circuit motif, which is repeated throughout the network in both bottom-up and top-down pathways. Combined with a local Hebbian learning rule, the resulting model propagates learning signals and updates weights using only local interactions between neurons. We show theoretically that our learning scheme can exactly recover the backpropagation update despite relying solely on non-negative error signals. Empirically, beyond satisfying stronger biological constraints, the on-off architecture learns efficient representations, yielding substantial gains over comparable vanilla networks on the Tiny ImageNet benchmark. These results demonstrate that effective learning can emerge from biologically plausible mechanisms without requiring mixed-sign signals, providing a step toward more realistic models of neural computation.
\end{abstract}


\section{Introduction}
    
Deep learning has become an important framework for systems neuroscience, where architectures, objective functions, and learning rules serve as computational hypotheses about neural circuits in the brain \citep{richards2019deep}.  Backpropagation is the central mechanism behind modern deep learning, but a direct implementation of its computations in biological neural circuits remains difficult to reconcile with known properties of the brain \citep{rumelhart1986learning, whittington2019theories, lillicrap2020backpropagation}. Such an implementation would require mechanisms for global objective optimization, non-local credit assignment, closely matched forward and backward pathways, and the propagation of signed error signals through the network. These challenges have motivated a large literature on biologically plausible learning, including predictive coding, equilibrium propagation, feedback alignment, target propagation, and dendritic learning models \citep{lee2015difference, lillicrap2016random, scellier2017equilibrium, whittington2017approximation, guerguiev2017towards, sacramento2018dendritic, payeur2021burst, hinton2022forward, abel2024biologically, oliviers2026bidirectional}. A common goal across these approaches is to preserve the computational benefits of gradient-based credit assignment while replacing the mechanism with local, biologically plausible computations.

A particularly fundamental obstacle is the representation of signed quantities. In artificial networks, activations, errors, gradients, and synaptic weights are numerical variables that can be positive or negative. In biological circuits, however, neuronal activity is non-negative: neurons fire at non-negative rates and cannot directly transmit negative activity \citep{lillicrap2020backpropagation}. In addition, neurons are constrained by Dale's law: each neuron has a fixed excitatory or inhibitory identity, so its outgoing synaptic effects share the same sign \citep{strata1999dale}. That is, a single neuron cannot exert both positive and negative outgoing influence. Thus, a biologically plausible learning mechanism must address not only how to compute credit assignment locally, but also how to do so in a Dale-constrained network: signed activations and signed learning signals must be represented using positive neural activity and fixed-sign synaptic connectivity.

Existing work addresses parts of this problem, but not the full combination of constraints. Predictive coding and related models show that local error units can approximate, and in some settings recover, backpropagation-like updates \citep{whittington2017approximation, song2020can, salvatori2022reverse}. Recent refinements further impose non-negative firing rates on predictive-coding error neurons \citep{alonso2021tightening}, but it still does not enforce Dale's law: the same neuron may still mediate both positive and negative outgoing effects unless its synaptic influence is sign-constrained. In contrast, Daleian neural networks explicitly enforce excitatory/inhibitory structure, which was shown to induce strong expressive power and robustness \citep{barranca2022functional, haber2022computational}, but are typically trained with backpropagation or other gradient-based optimization procedures. Thus, despite significant contributions made by recent models, the problem of performing effective supervised credit assignment using local biological mechanisms is still unresolved.   

One possible resolution is to represent signed learning signals through pairs of non-negative neural populations rather than through individual mixed-sign error units. \citet{lillicrap2020backpropagation} argue that cortical feedback may support credit assignment by inducing neural activities whose differences approximate error signals. Recent experimental evidence from \citet{francioni2026vectorized} further supports backpropagation-like learning in the brain, reporting vectorized dendritic feedback signals with opposing positive and negative contributions to learning, as well as task variables expressed through differences between positive populations. Together, these findings suggest that signed error signals may emerge from structured interactions between non-negative neural activities, but a concrete neural circuit architecture and local learning mechanism for implementing this principle remain open.

In this work, we introduce an \textit{on-off neural circuit motif} that is composed of excitatory and inhibitory neurons. The circuit produces two non-negative output channels corresponding to positive and negative components of a represented quantity. This motif is repeated throughout the network in a bottom-up and top-down structure \citep{gilbert2013top}, so that both feedforward activations and feedback error signals are represented by paired non-negative channels while preserving excitatory/inhibitory separation.

We further propose a local Hebbian learning update that combines bottom-up activity with paired top-down feedback signals. One top-down population contributes positive synaptic updates, and another population contributes negative updates, so the effective signed learning signal is recovered through their combined influence. Thus, the model implements credit assignment through effective signed signals that emerge from paired positive neural populations while preserving fixed excitatory and inhibitory circuit structure. We show theoretically that, under symmetrical bottom-up and top-down connectivity, the proposed learning scheme recovers the backpropagation update despite propagating only non-negative error channels. Empirically, we show that the model learns effectively on image-classification benchmarks and can outperform comparable standard architectures,  suggesting that the on-off representation can provide computational benefits beyond its biological motivation.

\section{Background}
\label{sec:background}


\paragraph{Core biological constraints.}
We focus on three constraints that are central to biological neural computation. First, neuronal activity is non-negative: neurons communicate through firing rates or spikes, and therefore cannot directly represent negative scalar values. Second, synaptic influence is sign-constrained. Under Dale's law, each neuron has a fixed excitatory or inhibitory identity, implying that its outgoing synaptic effects share the same sign. Third, synaptic plasticity is local: a synaptic update should depend on quantities available at or near the synapse, rather than on globally computed gradients. Together, these constraints rule out a direct biological implementation of standard artificial neurons, which rely on signed activations and sign-unconstrained weights.

\paragraph{Biologically plausible credit assignment.}
A large body of work seeks to replace backpropagation with more biologically plausible mechanisms while preserving the benefits of gradient-based credit assignment. Predictive coding models use local prediction errors and recurrent inference dynamics to approximate gradient-based learning \citep{rao1999predictive, whittington2017approximation}, with later variants recovering exact backpropagation updates under additional assumptions and architectural modifications \citep{song2020can, salvatori2022reverse}, and other refinements imposing non-negative firing rates on predictive-coding neurons \citep{alonso2021tightening}. However, such modifications have been criticized for reducing biological plausibility \citep{rosenbaum2022relationship, golkar2022constrained}. Equilibrium propagation computes local updates from changes in network equilibria \citep{scellier2017equilibrium, ernoult2019updates, millidge2020activation}. Feedback alignment relaxes the weight-transport problem by using separate feedfroward and feedback weights \citep{lillicrap2016random, nokland2016direct, akrout2019deep, webster2020learning}, while target propagation transmits target activations for the forward path rather than gradients \citep{lee2015difference, meulemans2020theoretical, ernoult2022towards}. Dendritic learning models separate feedforward and feedback information into distinct neuronal compartments and  focus on local mechanisms \citep{guerguiev2017towards, sacramento2018dendritic, payeur2021burst}. Hebbian-based learning methods provide local mechanisms in which synaptic updates depend on presynaptic and postsynaptic activity \citep{hebb2005organization, lagani2021hebbian, abel2024biologically}. Complementary forward-only approaches avoid an explicit backward pass and instead learn representations through unsupervised, contrastive, or layerwise objectives \citep{hinton2022forward, chen2025self}. However, most of these approaches still do not enforce Dale's law, and rely on mixed-sign error signals, sign-unconstrained synaptic effects, or both.

\paragraph{Daleian networks.}
A complementary line of work studies artificial networks that explicitly enforce Dale's law by separating excitatory and inhibitory units. Dale-constrained recurrent networks have been used to model task-dependent neural dynamics \citep{song2016training}, and other work has analyzed the dynamical consequences of Daleian structure, including balanced dynamics and robustness \citep{barreiro2017symmetries, barranca2022functional}. More recently, \citet{haber2022computational} showed that Daleian networks retain strong expressive power, can approximate non-Daleian computations, and can be more robust to synaptic noise. They also study learning based on tuning neuron-level parameters, such as incoming and outgoing synaptic scaling and neuronal bias, rather than directly updating individual synapses. This provides an important step toward learning in Daleian networks, but it does not solve supervised credit assignment in deep networks: it does not propagate error signals through layers under Dale-constrained circuitry, nor does it provide a local synaptic rule that recovers gradient-like updates. More broadly, most Dale-constrained architectures impose the Daleian structure only on the forward computation, while training still relies on backpropagation or related gradient-based procedures whose biological implementation would violate the same constraints. Thus, they pair a biologically realistic architecture with a biologically implausible learning mechanism.

\paragraph{Paired on-off channels in visual processing}
Paired, complementary channels are a recurring principle in biological vision. In the retina, ON and OFF pathways respond preferentially to light increments and decrements \citep{kuffler1953discharge, schiller1992and}. In primary visual cortex, simple-cell receptive fields contain ON and OFF subregions, often organized through interactions between excitation and inhibition \citep{awang1962receptive, ringach2004mapping, liu2010intervening}. These findings do not establish the exact circuit proposed here, but they support the biological plausibility of representing visual variables through structured complementary populations. 


\section{Model}
\label{sec:model}

We introduce a neural architecture that is Daleian in the following sense: (i) it represents signed quantities using only non-negative neuronal activity, (ii) all synapses have a fixed sign synaptic influence, and (iii) all synapses from a cell to all its target neurons are of the same sign, excitatory or inhibitory. The central component is a repeated \textit{on-off neural circuit motif}: a small excitatory--inhibitory circuit that maps two non-negative inputs to a pair of non-negative output channels. The two channels encode positive and negative components of an underlying signed quantity, without requiring any individual neuron to carry negative activity. The same motif is used throughout the network in both the bottom-up (BU) stream, which computes activations, and the top-down (TD) stream, which propagates feedback signals for learning. Together with a local Hebbian update rule, this construction enables signed credit assignment under non-negative activity and Dale-constrained connectivity.   

\subsection{On-off neural circuit motif}

\begin{figure}[t]
  \centering
  \includegraphics[width=0.85\linewidth]{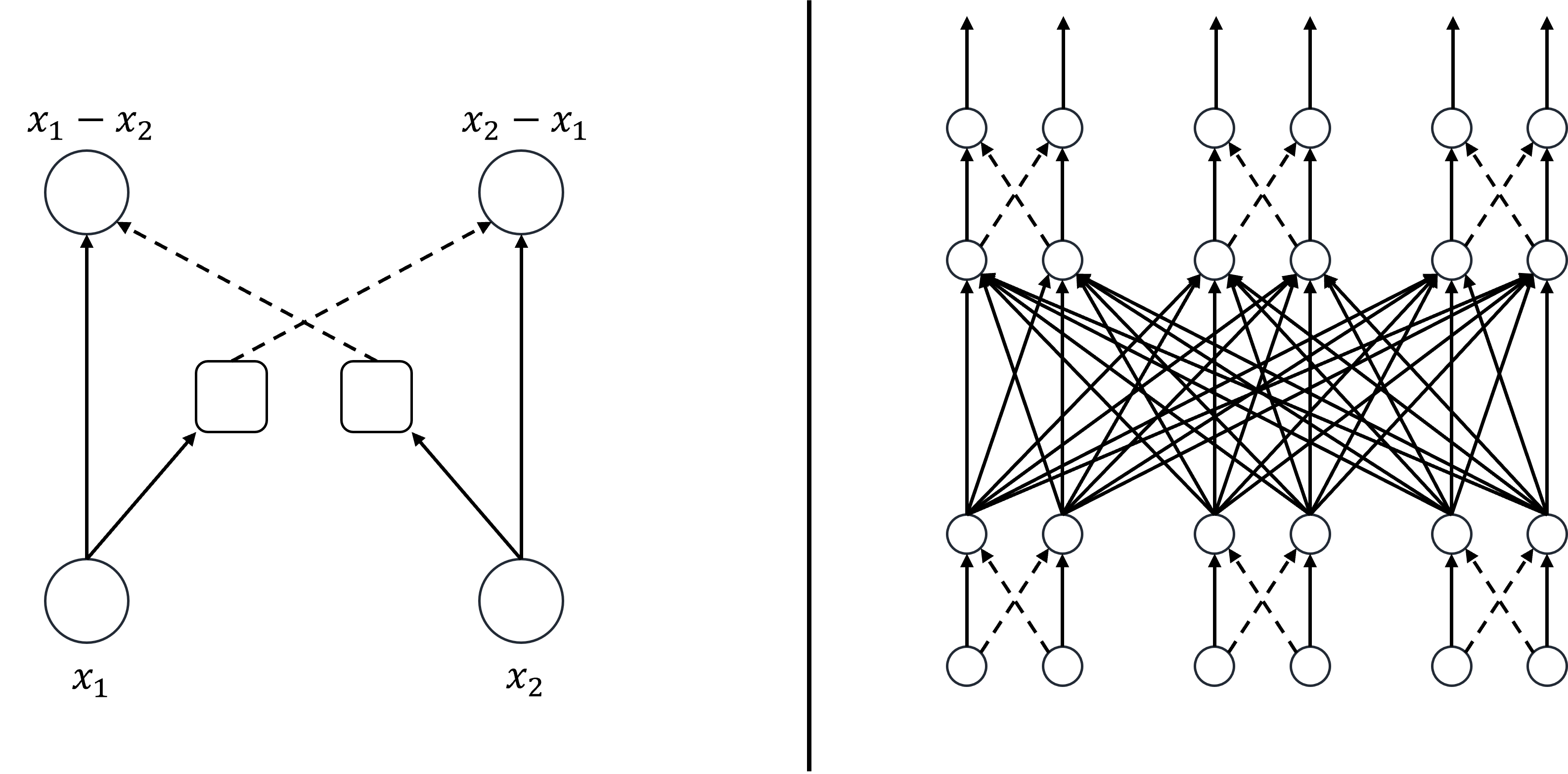}
    \caption{
    \textbf{On-off neural circuit motif.}
    \textbf{Left:} The basic motif receives two non-negative inputs, $x_1$ and $x_2$, and produces two non-negative output channels. Circles denote excitatory neurons and rectangles denote inhibitory neurons; solid and dashed arrows denote excitatory and inhibitory synapses, respectively. When all synaptic strengths are equal to one, the circuit encodes the signed difference between the inputs: the On channel fires when $x_1 > x_2$, while the Off channel fires when $x_2 > x_1$. Thus, a signed value is represented by two non-negative channels, with at most one active channel.
    \textbf{Right:} Networks are constructed by repeating this motif across layers. Inter-layer connectivity follows the standard pattern for the non-negative channels, while synaptic weights remain sign-constrained.
    }
    \label{fig:on_off_motif}
\end{figure}

The basic motif receives two non-negative inputs, $x_1,x_2 \geq 0$, and produces two non-negative output channels, denoted $y^+$ and $y^-$. The circuit contains excitatory input and output neurons, together with intermediate inhibitory neurons that mediate cross-channel inhibition (Figure~\ref{fig:on_off_motif}). Each input excites one output channel and, through an inhibitory inter-neuron, suppresses the opposite output channel. To illustrate the computation performed by the motif, consider the special case in which all internal synaptic magnitudes are equal to one. The circuit then computes the signed contrast between its inputs: $d = x_1 - x_2$. The two output channels are obtained by applying complementary rectifying nonlinearities:
\begin{equation}
\begin{aligned}
    y^+ &= \mathrm{ReLU}^{+}_{\theta}(d)
         \coloneqq \max(d-\theta,0), \\
    y^- &= \mathrm{ReLU}^{-}_{\theta}(d)
         \coloneqq \max(-d-\theta,0),
\end{aligned}
\end{equation}

where $\theta \geq 0$ is a threshold and ReLU is the Rectified Linear Unit function. 
Thus, the On channel is active when $x_1$ sufficiently exceeds $x_2$, the Off channel is active when $x_2$ sufficiently exceeds $x_1$, and neither channel is active when the contrast magnitude is below the threshold. Therefore, at most one output channel is active. Although the network itself propagates only the non-negative channels $y^+$ and $y^-$, the signed value represented by the motif is encoded by the pair: A positive difference $d$ is propagated through $y^+$, whereas a negative difference is propagated with the same positive magnitude through the opposite channel $y^-$, thereby encoding the sign through channel identity.

\paragraph{General sign-constrained motif.}
More generally, the internal synaptic magnitudes of the motif need not be fixed to one. They can be learned while preserving their excitatory or inhibitory signs. In this case, the motif computes a soft weighted difference of the input neurons. This relaxes the hard unit-magnitude wiring while maintaining the excitatory/inhibitory structure of the motif. 

\subsection{Network architecture}

A full network is constructed by replacing each hidden mixed-sign neuron in a standard architecture with an on-off motif, preserving the original connectivity pattern while expanding each scalar activity into a pair of non-negative channels.

Each motif outputs two channels, so a layer with $n$ on-off units contains $2n$ non-negative activity channels. Inter-layer connections follow the connectivity pattern of the original architecture and are implemented as excitatory connections, since the output neurons of each on-off motif are excitatory, as illustrated in Figure~\ref{fig:on_off_motif}. Negative influences are not carried by negative inter-layer weights, but are mediated by the complementary on-off channels and the inhibitory circuitry within each motif. Learning therefore adjusts synaptic magnitudes while preserving excitatory or inhibitory identity, enforcing Dale-constrained connectivity while allowing flexible learned transformations.

\subsection{Top-down propagation}

The TD stream propagates feedback signals using the same on-off motif structure as the BU stream. For each BU motif, we associate two corresponding TD motifs, one for the BU On channel and one for the BU Off channel. Each TD motif contains two non-negative neuron populations, corresponding to positive and negative synaptic updates. 

TD propagation flows in the opposite direction to the BU, and combines three operations. First, lateral BU-TD connectivity gates the TD neural activity according to the active BU channel, suppressing the TD pathway associated with the inactive channel. This bidirectional gating mechanism is similar to the BU-TD connectivity proposed by \citet{abel2024biologically}, and plays a role analogous to the derivative of a rectifying nonlinearity in standard backpropagation. Second, TD activity is processed within each layer by the on-off motif and cross-population connectivity between positive and negative feedback populations. Finally, feedback is propagated to lower layers through TD inter-layer connections, separately within each TD population. Thus, signed feedback is represented by differences between paired non-negative TD populations while all neuronal activities remain non-negative. A complete schematic of the TD flow is provided in Appendix~\ref{app:td_flow}.

\subsection{Local Hebbian learning scheme}

\begin{figure}[t]
  \centering
  \includegraphics[width=0.4\linewidth]{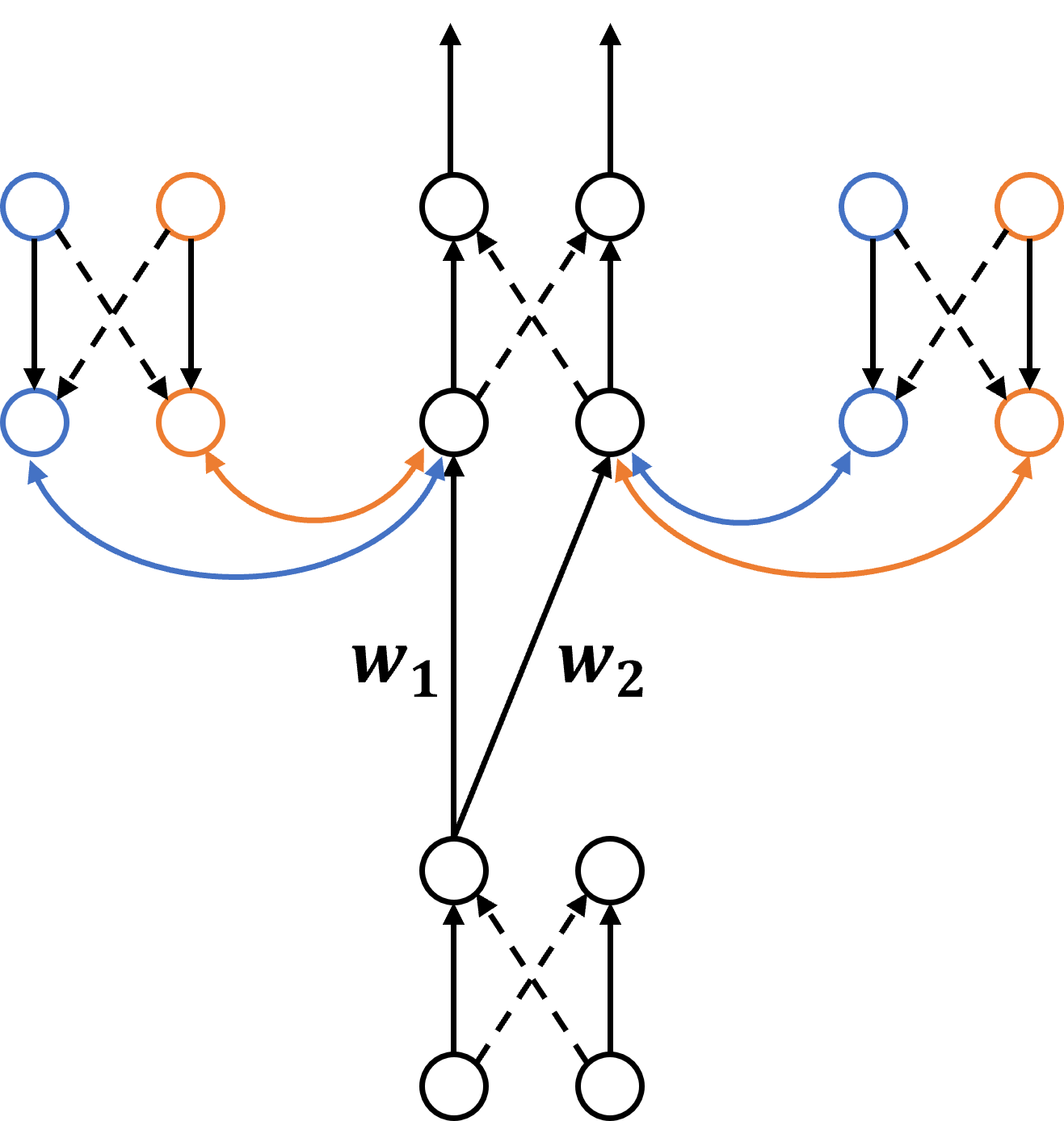}
    \caption{
    \textbf{Top-down streams and Hebbian learning.}
    The figure illustrates the update of inter-layer bottom-up weights connecting consecutive on-off layers. A bottom-up on-off motif receives feedback signals from two corresponding top-down motifs, associated with its On and Off channels. The top-down activities are non-negative, but are separated into channels corresponding to positive and negative synaptic updates (blue and orange). Learning is local: activity in the positive-update channel drives a Hebbian increase, whereas activity in the negative-update channel drives a Hebbian decrease, yielding an effective signed update proportional to the difference between paired top-down signals. For clarity, the figure shows only the connections relevant to the synaptic updates of a single bottom-up motif; a complete schematic is provided in Appendix~\ref{app:model_details}.
    }
    \label{fig:on_off_learning}
\end{figure}

Learning is implemented by a local Hebbian update rule that combines presynaptic BU activity with paired TD feedback signals. For a synapse with presynaptic activity $x_j$, learning consists of two Hebbian contributions: activity in the positive TD channel drives a Hebbian increase, whereas activity in the negative TD channel drives a Hebbian decrease: 
\begin{equation}
\begin{aligned}
    \Delta W_{ij} &= \eta x_j \bar{\partial}_i, \\
    \Delta W_{ij} &= -\eta x_j \underline{\partial}_i .
\end{aligned}
\end{equation}
Here, $\bar{\partial}_i,\underline{\partial}_i \geq 0$ denote the activities of the paired TD channels associated with the postsynaptic unit: $\bar{\partial}_i$ contributes to weight increases, while $\underline{\partial}_i$ contributes to weight decreases, see Figure~\ref{fig:on_off_learning}. Combining the two Hebbian contributions yields the effective synaptic update
\begin{equation}
    \Delta W_{ij}
    = \eta x_j(\bar{\partial}_i - \underline{\partial}_i).
\end{equation}
Thus, the effective signed feedback signal can be viewed as:
\begin{equation}
    \partial_i = \bar{\partial}_i - \underline{\partial}_i ,
\end{equation}
where the sign of the update is determined by the relative activity of two separated TD populations, rather than by a signed error neuron.

This rule is local: the update of a synapse depends only on the presynaptic BU activity neuron and TD feedback signals arriving at the corresponding postsynaptic channel. The synaptic sign itself remains fixed; learning modifies synaptic magnitudes while preserving excitatory or inhibitory identity.

TD synapses follow the same Hebbian principle. Each TD weight is modified using its own presynaptic TD activity together with the corresponding postsynaptic BU-channel activity and the opposing TD-channel signal. As a result, corresponding BU and TD synapses receive the same update. Therefore, if the weights of the two pathways are aligned before the update, they remain aligned afterward.

The proposed architecture, therefore, provides a mechanism for representing activations, propagating feedback, and updating synapses using only non-negative neuronal activity and fixed-sign circuitry. In the next section, we analyze when this local learning rule produces gradient-based updates.

\section{Theoretical analysis}
\label{sec:theory}

We analyze the learning dynamics of the proposed model and show that it can produce  gradient-based updates despite operating with non-negative neuronal activity and sign-constrained connectivity. The key mechanism enabling this result is the representation of signed signals via differences between paired non-negative channels.

\subsection{Backpropagation}
Consider a standard feedforward network, where $h_l$ denotes the activations in layer $l$, $W_l$ denotes the weights from layer $l-1$ to layer $l$, and $\sigma$ is the activation function. Backpropagation computes signed learning signals $\delta_l$ recursively:
\begin{equation}
    \delta_{l-1}
    =
    \sigma'(h_{l-1}) W_l^T \delta_l ,
    \label{eq:bp_recursion}
\end{equation}
and the corresponding weight update is proportional to
\begin{equation}
    \Delta W_l
    \propto
    -\frac{\partial L}{\partial W_l}
    =
    \delta_l h_{l-1}^T .
    \label{eq:bp_update}
\end{equation}
The central challenge for a biological implementation is that $\delta_l$ is a signed quantity that propagates backward through the network. 

\subsection{Propagation through the on-off top-down stream}

In the on-off model, a signed learning signal is represented by two non-negative TD populations with opposing instructive roles: one drives increases of the corresponding BU synapses, whereas the other drives their decrease. Denoting their activities by $\bar{\partial}_l$ and $\underline{\partial}_l$, the effective signed learning signal is:
\begin{equation}
    \partial_l
    =
    \bar{\partial}_l - \underline{\partial}_l,
    \qquad
    \bar{\partial}_l,\underline{\partial}_l \geq 0 .
    \label{eq:td_readout}
\end{equation}
Here, $\partial_l$ is only used for mathematical analysis; it is not transmitted as signed neuronal activity.

\paragraph{Theorem.}
Assume that (i) the BU network consists of linear operations, such as fully connected, locally connected, convolutional layers, or residual connections, followed by ReLU nonlinearities, and (ii) the BU and TD weights are aligned in their weights. Let $\delta_l$ denote the standard backpropagation error signal at layer $l$. Then the effective TD signal satisfies the same recursion as $\delta_l$:
\[
    \bar{\partial}_l - \underline{\partial}_l = \delta_l.
\]
Consequently, the on-off Hebbian learning rule yields the same weight update as backpropagation.

\paragraph{Proof sketch.}
A full proof is provided in Appendix~\ref{app:gradient_recovery}. Let $W_l^+$ and $W_l^-$ denote the corresponding TD inter-layer weights for the positive and negative channels. The TD propagation satisfies
\begin{equation*}
    \bar{\partial}_{l-1}
    =
    \mathbf{Gate}_{h_{l-1}}  W_l^{+T} \bar{\partial}_l,
    \qquad
    \underline{\partial}_{l-1}
    =
    \mathbf{Gate}_{h_{l-1}}  W_l^{-T} \underline{\partial}_l .
\end{equation*}
Here, $ \mathbf{Gate}_{h_{l-1}}$ denotes the diagonal gating matrix induced by the BU activations $h_{l-1}$. Under the symmetric channel connectivity of the on-off model, $W_l^T = W_l^+ = W_l^-$, taking the difference between the two propagated channels gives
\begin{equation*}
    \bar{\partial}_{l-1}
    -
    \underline{\partial}_{l-1}
    =
      \mathbf{Gate}_{h_{l-1}} W_l^T(\bar{\partial}_l-\underline{\partial}_l).
    \label{eq:td_difference_propagation}
\end{equation*}
Assume that $\delta_l = \bar{\partial}_l - \underline{\partial}_l$ holds in layer $l$. Since the gating operation corresponds to the derivative of the ReLU activation, we obtain
\begin{equation}
    \bar{\partial}_{l-1}
    -
    \underline{\partial}_{l-1}
    =
     \sigma'(h_{l-1}) W_l^T(\delta_l)
    =
     \delta_{l-1}.
    \label{eq:td_bp_recursion}
\end{equation}
Therefore, by induction, if the equality holds at the top layer, it holds for all layers.

Finally, the Hebbian on-off learning rule gives:
\begin{equation}
    \Delta W_l
    =
    \eta(\bar{\partial}_l-\underline{\partial}_l)h_{l-1}^T
    =
    \eta\delta_l h_{l-1}^T,
\end{equation}
which matches the gradient descent update in Eq.~\ref{eq:bp_update}.

\paragraph{Separate BU and TD weights.}
The exact result assumes aligned BU and TD connectivity. In practice, the two pathways are parameterized separately. However, as described in Section~\ref{sec:model}, corresponding BU and TD synapses receive matched local updates. Therefore, if the two pathways are aligned, this alignment is preserved throughout learning; if they are initialized close to alignment, the TD computation approximates the backpropagation recursion.

\section{Experiments}
\label{sec:experiments}

We evaluate the proposed on-off model in two regimes. First, controlled experiments on \textit{MNIST} \citep{lecun1998gradient}, \textit{Fashion-MNIST} \citep{xiao2017fashion}, and \textit{CIFAR-10} \citep{krizhevsky2009learning} examine three questions: (i) whether the local Hebbian update matches backpropagation in the symmetric setting, (ii) whether the on-off representation is competitive with standard ReLU networks of comparable size, and (iii) how performance depends on bottom-up/top-down synaptic alignment and the parameterization of the on-off internal motif. Second, experiments on \textit{Tiny ImageNet} \citep{le2015tiny} evaluate whether the on-off architecture and local Hebbian learning scale to a more challenging visual benchmark.

\subsection{Comparative experiments}

\paragraph{Setup.}
We evaluate fully connected two-layer networks. For MNIST and Fashion-MNIST, we use 250 hidden on-off units per layer and train for 20 epochs. For CIFAR-10, we use 500 hidden on-off units per layer and train for 50 epochs. All experiments use batch size 64 and learning rate $10^{-2}$, without extensive hyperparameter tuning. We report the performance at the best epoch during training, averaged across 10 random seeds.

We compare against two ReLU MLP baselines, trained with standard backpropagation. The first has the same number of effective hidden units as the on-off model, replacing every on-off unit motif with standard artificial neuron, while the second doubles the width to match the actual number of hidden neurons, since each on-off unit contains two non-negative channels.

\paragraph{Variants.}
We evaluate several on-off variants. \textit{BP} trains the on-off architecture with standard backpropagation and serves as a gradient-based reference. \textit{Sym} uses the proposed Hebbian learning scheme with symmetric BU and TD weights. \textit{Asym} initializes the two pathways far from alignment, using approximately orthogonal corresponding BU and TD weights. \textit{Weak Sym} initializes the pathways close to, but not exactly symmetric. \textit{Noisy} extends the weakly symmetric setting by adding independent noise to the local updates.

We also vary the internal on-off motif. The default fixed-magnitude motif uses internal weights of magnitude $1$, implementing the difference operation. In \textit{Learned Unit}, the internal motif weights are initialized to $1$ and then learned under fixed sign constraints using the proposed Hebbian scheme. In \textit{Shared Learned Unit}, the internal motif weights are also learnable and initialized to one, but a single set of internal motif weights is shared across all units within each layer. Full experimental details, including the definitions of all model variants, are provided in Appendix~\ref{app:experimental_details}.

\begin{table}[t]
\centering
\caption{
Controlled experiments on MNIST, Fashion-MNIST, and CIFAR-10. We report test accuracy (\%) averaged over 10 random seeds, with standard deviation.
}
\label{tab:controlled_results}
\begin{tabular}{lccc}
\toprule
Method & MNIST & Fashion-MNIST & CIFAR-10 \\
\midrule
Vanilla MLP & $97.666 \pm 0.066$ & $87.891 \pm 0.178$ & $54.057 \pm 0.274$ \\
Vanilla MLP $\times 2$ & $97.784 \pm 0.066$ & $88.077 \pm 0.171$ & $55.510 \pm 0.423$ \\
\midrule
On-Off BP & $97.792 \pm 0.067$ & $88.100 \pm 0.238$ & $56.432 \pm 0.372$ \\
On-Off Sym & $97.796 \pm 0.074$ & $88.147 \pm 0.229$ & $56.483 \pm 0.318$ \\
On-Off Weak Sym & $97.795 \pm 0.079$ & $87.888 \pm 0.268$ & $56.394 \pm 0.377$ \\
On-Off Asym & $95.244 \pm 0.492$ & $79.468 \pm 1.814$ & $41.707 \pm 2.071$ \\
On-Off Noisy & $97.791 \pm 0.094$ & $87.734 \pm 0.322$ & $56.370 \pm 0.578$ \\
On-Off Learned Unit & $97.765 \pm 0.054$ & $87.867 \pm 0.276$ & $56.434 \pm 0.443$ \\
On-Off Shared Learned Unit & $\mathbf{98.027 \pm 0.051}$ & $\mathbf{88.276 \pm 0.333}$ & $\mathbf{56.484 \pm 0.346}$ \\
\bottomrule
\end{tabular}
\end{table}

\paragraph{Results.}
Table~\ref{tab:controlled_results} shows that the on-off model learns effectively on all benchmarks. The symmetric Hebbian model matches backpropagation on the same architecture, with differences well within the standard deviation between seeds, supporting the theoretical analysis. The alignment variants further show that exact BU--TD symmetry is not required: both the Weak Sym and Noisy variants remain close to the symmetric and backpropagation references, indicating that approximate alignment and noisy local updates are sufficient for effective credit assignment. In contrast, initializing the pathways far from alignment harms performance, especially on CIFAR-10, suggesting that performance depends on the initial alignment regime rather than on exact equality between corresponding BU and TD weights. The model is also competitive with standard ReLU networks; on CIFAR-10, on-off Sym reaches $56.483\%$, outperforming both the matched-effective-width MLP ($54.057\%$) and the matched-neuron-count MLP ($55.510\%$).

The learned-motif variants show that the on-off circuit does not need to enforce a fixed hard difference with weights $\pm1$. The model learns effectively when the internal motif weights are themselves learned under sign constraints, starting from the unit-magnitude initialization. The Shared Learned Unit variant gives the best results on all three datasets, suggesting that sharing the learned on-off internal motif computation within a layer may provide a useful regularization effect.

\subsection{Scaling to Tiny ImageNet}

\paragraph{Setup.}
We next evaluate scalability on Tiny ImageNet \citep{le2015tiny}, a 200-class image classification benchmark that is substantially more challenging than CIFAR-10. This benchmark is a useful stress test for biologically motivated learning methods: for example, as reported in the Self-Contrastive Forward--Forward (SCFF) study, the Hebbian-based approach remains competitive on simpler benchmarks such as CIFAR-10, but scales poorly to Tiny ImageNet and suffers large performance drops \citep{chen2025self}. To compare with this recent evaluation, we follow the SCFF Tiny ImageNet experimental setting, using a similar five-convolution-layer architecture followed by a linear classifier. This yields a comparison between models with similar size and abstract structure. The on-off model is trained with the proposed local Hebbian rule in its symmetric BU-TD setting. We additionally compare against two vanilla convolutional baselines: one with the same number of effective hidden units, where each on-off unit is replaced by a standard ReLU neuron, and one with twice the number of hidden channels, matching the number of hidden neurons/channels in the on-off model. Full architectural and optimization details are provided in Appendix~\ref{app:experimental_details}.

\begin{table}[t]
\centering
\caption{
Tiny ImageNet results. We report top-1 and top-5 accuracy (\%). Results are averaged over 5 random seeds. Reference results are reported from \citep{chen2025self}, '--' means no reported results.
}
\label{tab:tinyimagenet_results}
\begin{tabular}{lcc}
\toprule
Method & Top-1 & Top-5 \\
\midrule
Hebb-based \citep{lagani2021hebbian} & -- & $37.0$ \\
DFA \citep{webster2020learning} & $32.1$ & -- \\
SCFF \citep{chen2025self} & $35.7$ & $59.8$ \\
\midrule
Vanilla Conv & $35.58 \pm 0.44$ & $60.69 \pm 0.32$ \\
Vanilla Conv $\times 2$ & $37.10 \pm 0.57$ & $62.64 \pm 0.33$ \\
On-Off (Ours)& $\mathbf{42.31 \pm 0.41}$& $\mathbf{67.70 \pm 0.31}$ \\
\bottomrule
\end{tabular}
\begin{flushleft}
\end{flushleft}
\end{table}

\paragraph{Results.}
Table~\ref{tab:tinyimagenet_results} reports final-epoch test accuracy, showing that the on-off model scales beyond the small scale fully connected setting. On the same five-convolution-layer scale as SCFF, on-off Sym reaches $42.31\%$ top-1 and $67.70\%$ top-5 accuracy, outperforming SCFF by a large margin. It also outperforms both vanilla convolutional baselines: the matched-effective-width model reaches $35.58\%$ top-1, and the matched-neuron-count $\times 2$ model reaches $37.10\%$. Thus, the improvement cannot be explained simply by the doubled-channel representation of the on-off units: the on-off model outperforms a vanilla network trained without biological constraints even when the two models have the same number of hidden neurons/channels. This suggests that the paired-channel architecture provides an efficient representation relative to vanilla networks of comparable size. To further analyze this representation, Appendix~\ref{app:activation_visualization} visualizes the learned activations of opposing on-off channels.


\section{Discussion}
\label{sec:discussion}
 
We introduced an on-off architecture for biologically constrained learning, where neural activations and learning signals are represented through paired non-negative channels. This provides a concrete mechanism for a long-standing issue in biologically plausible credit assignment: error signals must be able to increase or decrease synaptic weights, yet biological neurons communicate through non-negative activity and obey excitatory/inhibitory constraints. In our model, two TD populations play opposing instructive roles, with one driving Hebbian increases and the other driving decreases of synaptic weights. This is in line with biological observations of on-off and opponent channels in sensory systems in low-level parts of the visual system \citep{awang1962receptive, schiller1992and}, as well as in higher-level, e.g., in face and expression perception \citep{valentine1991unified, skinner2010anti}.

The model is also in line with theoretical proposals that feedback may induce error-like differences between neural activities \citep{lillicrap2020backpropagation}, and recent evidence for vectorized instructive signals in cortical dendrites \citep{francioni2026vectorized}. The structure of the model is relatively simple: the same on-off motif is repeated across the network and used in both BU and TD streams, composing signed signals using a two-channel circuit-level representation based on excitatory and inhibitory interactions. The experiments suggest that this construction is not merely a costly doubling of network neurons. Across benchmarks, the on-off model outperforms vanilla networks of comparable size, trained with both biologically inspired methods and backpropagation. These results suggest that paired $\mathrm{ReLU}^{+}/\mathrm{ReLU}^{-}$ channels can provide an efficient representation for learning in biological networks and possibly also in artificial models.

\paragraph{Limitations:} 
The current model focuses on one important  component, but it is still far from providing a full biologically plausible model for cortical learning, and significant research is still required. A second limitation is that the on-off model requires a substantial increase in the number of neurons, compared with models that use a single top-down network rather than two parallel networks used by the current model. However, it is consistent with the experimental estimate that the number of top-down cortical connections in primates is about twice that of bottom-up ones \citep{markov2014anatomy} and, as shown by the model, the two-network architecture can increase the model's performance.

\bibliographystyle{unsrtnat}
\bibliography{main}


\newpage

\appendix

\section{Technical appendices and supplementary material}

\subsection{Complete on-off model connectivity}
\label{app:model_details}

This appendix provides a complete schematic description of the connectivity and computational flow of the on-off model. The main text presents the basic on-off motif and the local update rule in simplified form. Here we describe how bottom-up motifs, top-down motifs, gating connections, and inter-layer weights are arranged in the full model.

\subsubsection{Bottom-up and top-down motifs}

Figure~\ref{fig:app_bu_td_overview} shows the full organization of two consecutive on-off layers. Each layer contains a single bottom-up (BU) on-off motif, through which feedforward activity propagates upward. For every BU motif, the model includes two corresponding top-down (TD) motifs: one associated with the BU On channel and one associated with the BU Off channel. These TD motifs propagate feedback signals downward.

Each TD motif is itself composed of two non-negative channels, corresponding to two TD populations of positive and negative synaptic updates. In the figure, the positive-update population is shown in blue and the negative-update population in orange. Thus, while every BU motif contains two output channels, its corresponding TD structure contains four channels, derived from two TD motifs. This organization allows the model to represent signed feedback signals using only non-negative neural activity.

\begin{figure}[ht]
    \centering
    \includegraphics[width=0.4\linewidth]{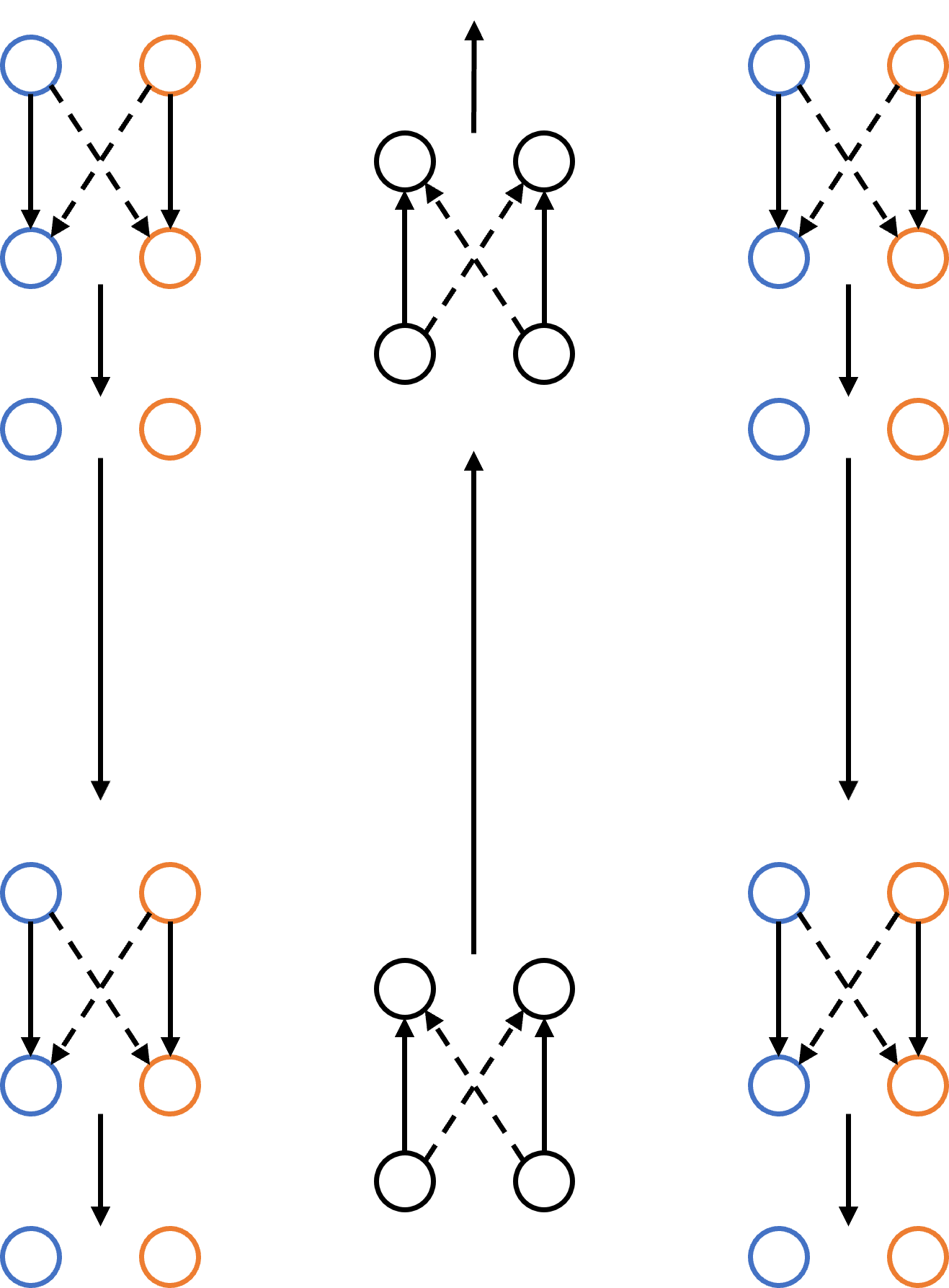}
    \caption{
    \textbf{Bottom-up and top-down motif organization.}
    The figure shows two consecutive on-off layers, each containing one bottom-up (BU) on-off motif and two corresponding top-down (TD) motifs. BU activity flows upward through the network. For each BU motif, one TD motif corresponds to the BU On channel and the other to the BU Off channel. Each TD motif contains two non-negative feedback channels: a positive-update population (blue) and a negative-update population (orange). TD activity flows downward and provides the feedback signals used for local Hebbian updates.
    }
    \label{fig:app_bu_td_overview}
\end{figure}

\subsubsection{Bottom-up connectivity}

The BU computation is illustrated in Figure~\ref{fig:app_bu_flow}. Within each on-off motif, the two output neurons are passed through the on-off rectifying nonlinearities, denoted by $\mathrm{ReLU}^{+}$ and $\mathrm{ReLU}^{-}$. As a result, at most one output neuron of each BU motif is active at a time. The active output channel represents the sign and magnitude of the encoded scalar quantity.

Inter-layer BU connectivity is applied over the individual output channels of the previous layer and the individual input neurons of the next layer. Thus, if a layer contains $n$ on-off motifs, it produces $2n$ non-negative output neurons, and a fully connected BU layer maps these $2n$ neurons to the input neurons of the motifs in the next layer. In this sense, the network preserves the standard fully connected structure over neurons, while grouping pairs of neurons into on-off motifs.

\begin{figure}[t]
    \centering
    \includegraphics[width=0.4\linewidth]{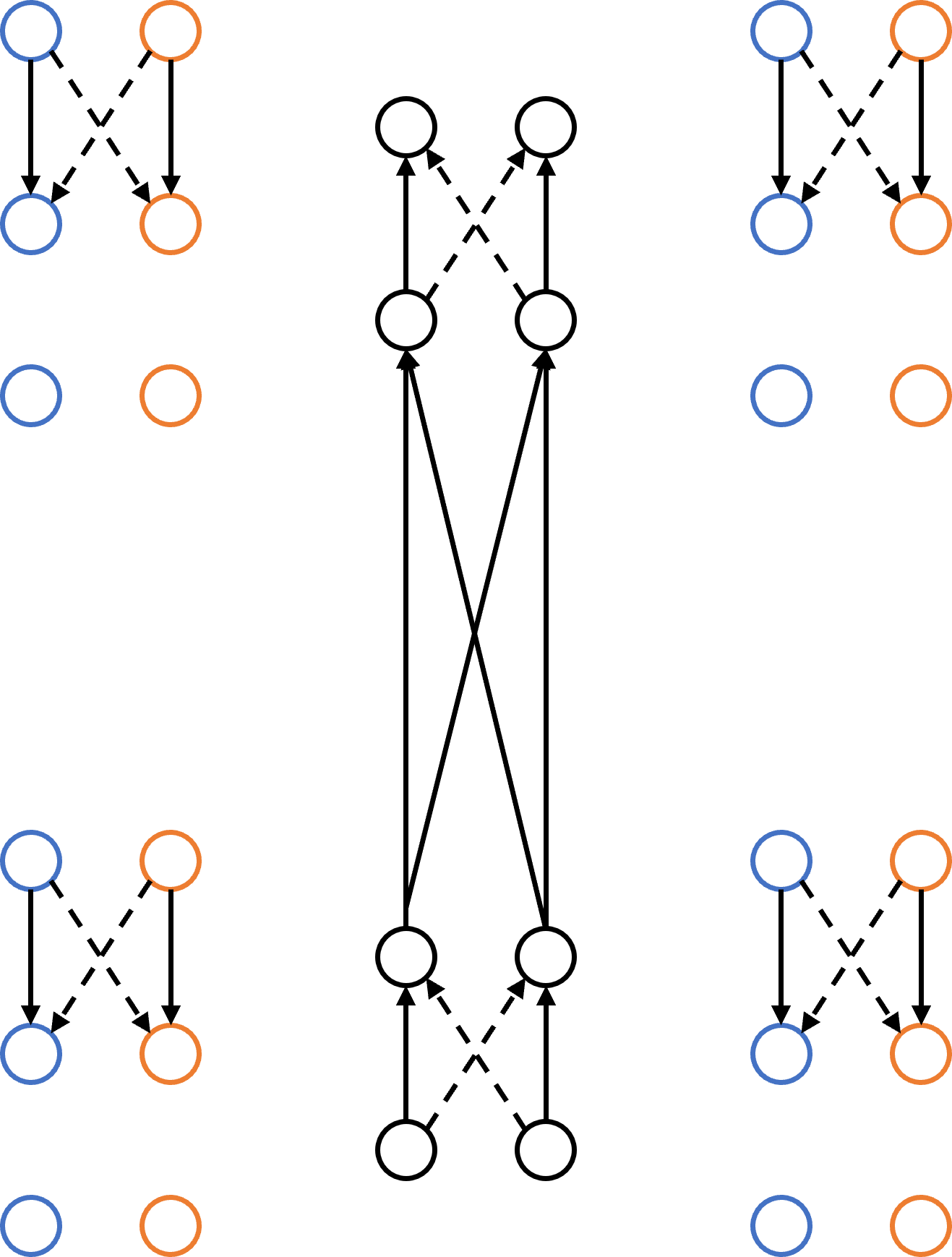}
    \caption{
    \textbf{Bottom-up connectivity and flow.}
    The BU stream propagates upward through consecutive on-off motifs. Each motif applies $\mathrm{ReLU}^{+}$ and $\mathrm{ReLU}^{-}$ nonlinearities, so that at most one output channel is active. The two output neurons of a lower-layer motif are connected to the input neurons of motifs in the next layer through standard inter-layer connectivity. Thus, inter-layer BU weights operate over individual non-negative channels, while pairs of channels jointly form on-off motifs.
    }
    \label{fig:app_bu_flow}
\end{figure}

\subsubsection{Top-down connectivity and gating}
\label{app:td_flow}

Figure~\ref{fig:app_td_flow} shows the TD connectivity. Although the full diagram is dense, the computation consists of four simple components.

First, TD propagation is gated by the corresponding BU activity. Since each BU motif has at most one active output channel, the TD motif corresponding to the inactive BU channel is suppressed. This gating operation is implemented through lateral BU--TD connections, shown by dashed colored arrows in the figure. Functionally, this operation plays the same role as the derivative of a rectifying nonlinearity: feedback is propagated only through the channel that participated in the BU computation.

Second, each TD motif applies the same on-off motif computation as the BU stream. The motif computes a weighted difference between its non-negative input channels and applies the on-off rectifying nonlinearities. Therefore, TD signals are also represented by paired non-negative channels rather than signed neurons.

Third, the two TD motifs associated with the same BU motif interact through cross-channel connectivity. The figure includes four intermediate TD neurons between the upper-layer TD motifs and the lower-layer TD inputs. These neurons implement the cross-channel structure: signals can cross both between the On and Off BU-associated motifs and between the positive-update and negative-update TD populations. This cross-channel connectivity is what allows signed error information to be represented as the difference between paired non-negative TD channels.

Fourth, the intermediate TD neurons are connected to the inputs of the next lower TD layer through inter-layer TD connectivity. The positive-update TD population is connected to the corresponding positive-update population in the next layer, and the negative-update population is connected analogously. Consequently, the model contains three inter-layer weight sets: one for the BU stream, one for the positive-update TD stream, and one for the negative-update TD stream. Although the TD stream contains twice as many neurons as the BU stream, the positive and negative TD populations are connected separately; therefore, each TD inter-layer weight matrix has the same shape as the BU inter-layer weight matrix, up to transposition.

\begin{figure}[t]
    \centering
    \includegraphics[width=0.4\linewidth]{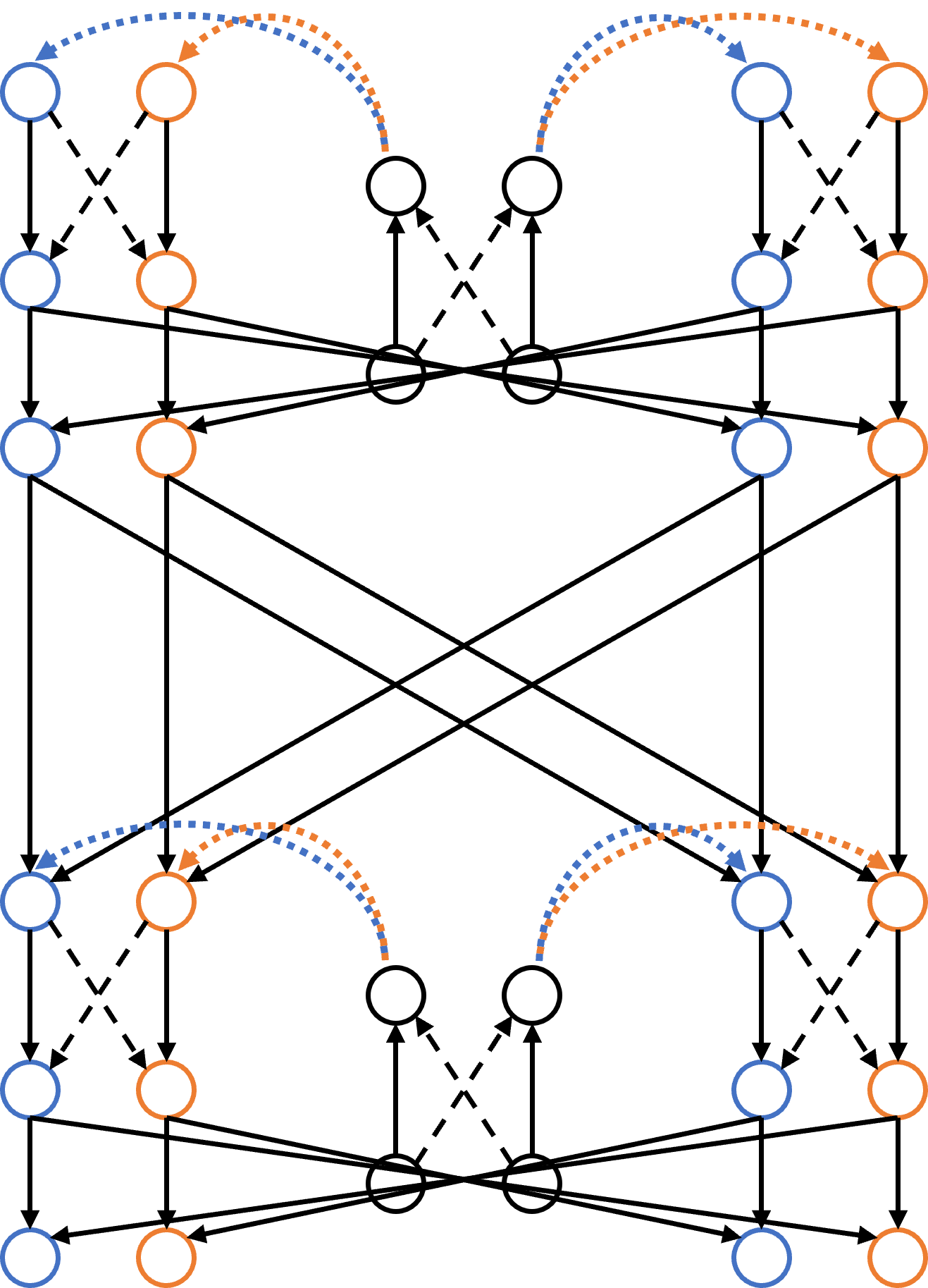}
    \caption{
    \textbf{Top-down connectivity and flow.}
    The TD stream propagates feedback downward through two TD motifs associated with each BU motif. The computation consists of four operations. First, lateral BU--TD connections gate TD activity according to the active BU channel. Second, each TD motif applies an on-off transformation, representing signed feedback through paired non-negative channels. Third, cross-channel TD connectivity links the two TD motifs associated with the same BU motif, allowing feedback to cross between On/Off-associated motifs and between positive-update and negative-update populations. Fourth, inter-layer TD connectivity propagates feedback to the next lower layer, separately for the positive-update and negative-update populations. This yields separate inter-layer weight sets for the BU stream, the positive-update TD stream, and the negative-update TD stream.
    }
    \label{fig:app_td_flow}
\end{figure}

\subsubsection{Computational flow of learning}

The full learning algorithm consists of three stages. First, the network performs a BU pass. The input is propagated upward through the on-off motifs, and each motif produces a pair of non-negative output channels. Due to the on-off nonlinearities, at most one channel in each motif is active.

Second, the network performs a TD pass. The output error is encoded into positive-update and negative-update TD channels and propagated downward through the TD stream. At each layer, lateral BU--TD gating restricts feedback to the TD motif corresponding to the active BU channel. The TD stream therefore carries feedback signals that are consistent with the BU computation while representing signed error information through paired non-negative channels.

Third, all weights are updated using a local Hebbian rule. For each inter-layer BU weight, the update depends on the product of the presynaptic BU activity and the relevant postsynaptic TD feedback channels. Activity in the positive-update TD channel contributes a Hebbian increase, whereas activity in the negative-update TD channel contributes a Hebbian decrease. The effective update is therefore proportional to the difference between paired TD channels. The corresponding TD weights are updated analogously, producing matched updates across BU and TD pathways. The same local learning principle is also applied to the intra-motif and intermediate TD connections, so all learnable synapses in the model are updated through local Hebbian interactions. 

The overall learning flow is summarized as:
\begin{enumerate}
    \item \textbf{Bottom-up pass:} propagate the input through BU on-off motifs.
    \item \textbf{Top-down pass:} propagate feedback through gated TD on-off motifs using positive-update and negative-update channels.
    \item \textbf{Local update:} update BU and TD weights using Hebbian products of local neural activities.
\end{enumerate}

\subsection{Experimental details}
\label{app:experimental_details}

All experiments were run on a single NVIDIA A10 GPU.

\subsubsection{Controlled experiments}

For the controlled experiments on MNIST, Fashion-MNIST, and CIFAR-10, we used two-layer fully connected networks. For MNIST and Fashion-MNIST, the on-off network contained two hidden layers with 250 on-off units per layer. For CIFAR-10, the network contained two hidden layers with 500 on-off units per layer. Since each on-off unit consists of two non-negative channels, these correspond to 500 and 1000 scalar hidden channels per layer, respectively.

All controlled experiments used threshold $\theta=0$ for the $\mathrm{ReLU}^{\pm}$ activation, included bias terms, and used Cross-Entropy loss. The internal on-off motif used fixed unit-magnitude difference weights unless evaluating the learnable motif variants. Training used stochastic gradient descent with learning rate $10^{-2}$, batch size 64, no weight decay, and no learning-rate decay. MNIST and Fashion-MNIST models were trained for 20 epochs, while CIFAR-10 models were trained for 50 epochs. We did not perform extensive hyperparameter tuning. Results are reported as test accuracy at the best epoch during training, averaged over 10 random seeds.

For vanilla baselines, we replaced the on-off units with standard ReLU neurons and trained using standard backpropagation. We used two baseline widths: one matching the number of effective hidden units in the on-off model, and one with twice the width, matching the number of scalar hidden neurons/channels.

\paragraph{Bottom-up/top-down alignment variants.}
The \textit{Sym} setting initializes corresponding BU and TD weights in exact alignment and applies matched Hebbian updates. The \textit{Weak Sym} setting initializes TD weights as noisy copies of the corresponding BU weights, using an independent elementwise Gaussian noise with standard deviation $0.01$. The \textit{Noisy} setting uses weakly symmetric initialization and adds independent elementwise noise to the TD updates, also with standard deviation $0.01$. The \textit{Asym} setting initializes the relevant BU and TD weight matrices independently at random; at the widths used here, such random high-dimensional weight vectors are typically close to orthogonal, yielding pathways far from alignment. The \textit{BP} reference trains the same on-off architecture directly with backpropagation.

\paragraph{on-off motif variants.}
In the default setting, the internal motif weights are fixed to unit magnitude, implementing the idealized difference operation. In the \textit{Learned Unit} setting, the internal motif weights are initialized to one and learned independently for each unit while preserving their signs. In the \textit{Shared Learned Unit} setting, the internal motif weights are initialized to one and learned, but a single set of motif weights is shared across all units within a layer.

\subsubsection{Tiny ImageNet experiments}

For Tiny ImageNet, we used a convolutional architecture following the scale of the Self-Contrastive Forward--Forward Tiny ImageNet setup. The model was trained on Tiny ImageNet-200 images of size $64\times64$. We report both top-1 and top-5 accuracy, averaged over 5 random seeds.

The on-off convolutional network consists of five convolutional stages followed by a linear output layer. Each convolutional stage applies a convolution over the current on-off channel representation, followed by the on-off motif and $\mathrm{ReLU}^{\pm}$ activation with threshold $\theta=0$. Max pooling is applied after stages 1, 2, and 5. The architecture is summarized in Table~\ref{tab:tiny_architecture}. The number of on-off pairs matches the corresponding vanilla SCFF-style architecture, while the actual number of channels is doubled due to the on-off representation.

\begin{table}[t]
\centering
\caption{
Tiny ImageNet convolutional architectures. The on-off model uses the same number of effective channel pairs as the vanilla SCFF-style architecture, but represents each channel by an on-off pair, doubling the actual number of scalar channels. Max pooling is applied after stages 1, 2, and 5.
}
\label{tab:tiny_architecture}
\begin{tabular}{lcccc}
\toprule
Stage & Kernel & Pool & Vanilla channels & on-off actual channels \\
\midrule
Conv1 & $5\times5$, pad 2 & yes & $3 \rightarrow 64$ & $6 \rightarrow 128$ \\
Conv2 & $3\times3$, pad 1 & yes & $64 \rightarrow 192$ & $128 \rightarrow 384$ \\
Conv3 & $3\times3$, pad 1 & no & $192 \rightarrow 384$ & $384 \rightarrow 768$ \\
Conv4 & $3\times3$, pad 1 & no & $384 \rightarrow 256$ & $768 \rightarrow 512$ \\
Conv5 & $3\times3$, pad 1 & yes & $256 \rightarrow 256$ & $512 \rightarrow 512$ \\
\bottomrule
\end{tabular}
\end{table}

For the vanilla convolutional baselines, each on-off unit is replaced by a standard ReLU channel. The matched-effective-width vanilla baseline uses the vanilla channel counts shown in Table~\ref{tab:tiny_architecture}. The matched-neuron-count baseline doubles these channel counts: $3\rightarrow128$, $128\rightarrow384$, $384\rightarrow768$, $768\rightarrow512$, and $512\rightarrow512$.

The on-off Tiny ImageNet model was trained in the symmetric BU/TD setting using the proposed local Hebbian rule. We used Cross-Entropy loss,    AdamW with learning rate $2.5\times10^{-4}$, weight decay $10^{-5}$, batch size 256, and 25 epochs. The learning rate followed a cosine schedule with warmup: warmup fraction $0.1$, peak learning-rate multiplier $1.1$, and final learning-rate multiplier $0.1$.

\subsection{Learned on-off representations}
\label{app:activation_visualization}

To further inspect the representations learned by the on-off architecture, we visualize activation maps from the convolutional model trained on Tiny ImageNet. We follow the activation-maximization visualization approach of \citet{zeiler2014visualizing}.  We focus on two representative on-off units from the second convolutional layer. For each unit, we score every image in the Tiny ImageNet test set by the maximum spatial activation of each channel:
\begin{equation}
    s_k^{+}(x) = \max_{h,w} h_k^{+}(x,h,w),
    \qquad
    s_k^{-}(x) = \max_{h,w} h_k^{-}(x,h,w),
\end{equation}
where $h_k^{+}$ and $h_k^{-}$ denote the post-$\mathrm{ReLU}^{\pm}$ On and Off activation maps. We then select the top 8 images for the On channel and the top 8 images for the Off channel, forming a set of 16 images per unit. For each selected image, we display both the On and Off activation maps together with the original image.

The normalized map is overlaid on the original image using a heatmap, so bright regions indicate where the corresponding channel responded most strongly. Because maps are normalized separately, the visualization indicates spatial selectivity within each image, rather than absolute activation magnitude across images or channels.

Figure~\ref{fig:on_off_activation_visualization} shows two representative examples of learned on-off units. The first unit exhibits a boundary-sensitive representation. Across the selected images, the On and Off channels respond to complementary parts of object boundaries. This behavior arises because the two channels are associated with opposite effective kernels: a visual pattern that produces a positive response in one channel produces the complementary response in the other. As a result, one channel tends to activate on one side or polarity of an edge, while the opposite channel activates on the reversed edge polarity. For example, when an object boundary contains a transition in one direction, the On channel may highlight that transition. Commonly the opposite transition occurs on the other whereas the Off channel highlights the opposite transition elsewhere in the image. This arises from basic properties of natural images where often a transition from background visual semantics to object are symmetric across spatial locations of the object. Together, the two channels provide a richer description of object shape than either channel alone, because the pair can represent both directions of local contrast within the same unit.

The second unit illustrates a different form of complementarity, related to color and foreground--background contrast. Here too, the paired \(\mathrm{ReLU}^{+}/\mathrm{ReLU}^{-}\) channels learn opposite effective kernels. One channel often responds strongly to regions associated with the object or foreground, especially where color or texture differs from the surrounding image context. The opposite channel tends to emphasize complementary regions, such as background structure or regions with the opposite color/contrast relationship. Thus, the pair separates complementary visual evidence along the same learned feature dimension, rather than simply duplicating the same detector twice.

These qualitative examples support the view that paired on-off channels are not redundant copies. Instead, the \(\mathrm{ReLU}^{+}/\mathrm{ReLU}^{-}\) construction encourages each unit to learn a feature together with its opposite-sign counterpart, yielding complementary responses that jointly encode a more complete visual feature: opposite edge polarity in the first example, and complementary color or foreground--background structure in the second. Importantly, this structure emerges from the on-off motif and the paired rectifying channels without explicit supervision encouraging complementary channel behavior. This is in line with biological observations of on-off and opponent channels in sensory systems in low-level parts of the visual system \citep{awang1962receptive, schiller1992and}, as well as in higher-level, e.g., in face and expression perception \citep{valentine1991unified, skinner2010anti}.

\begin{figure}[t]
  \centering
  \includegraphics[width=0.95\linewidth]{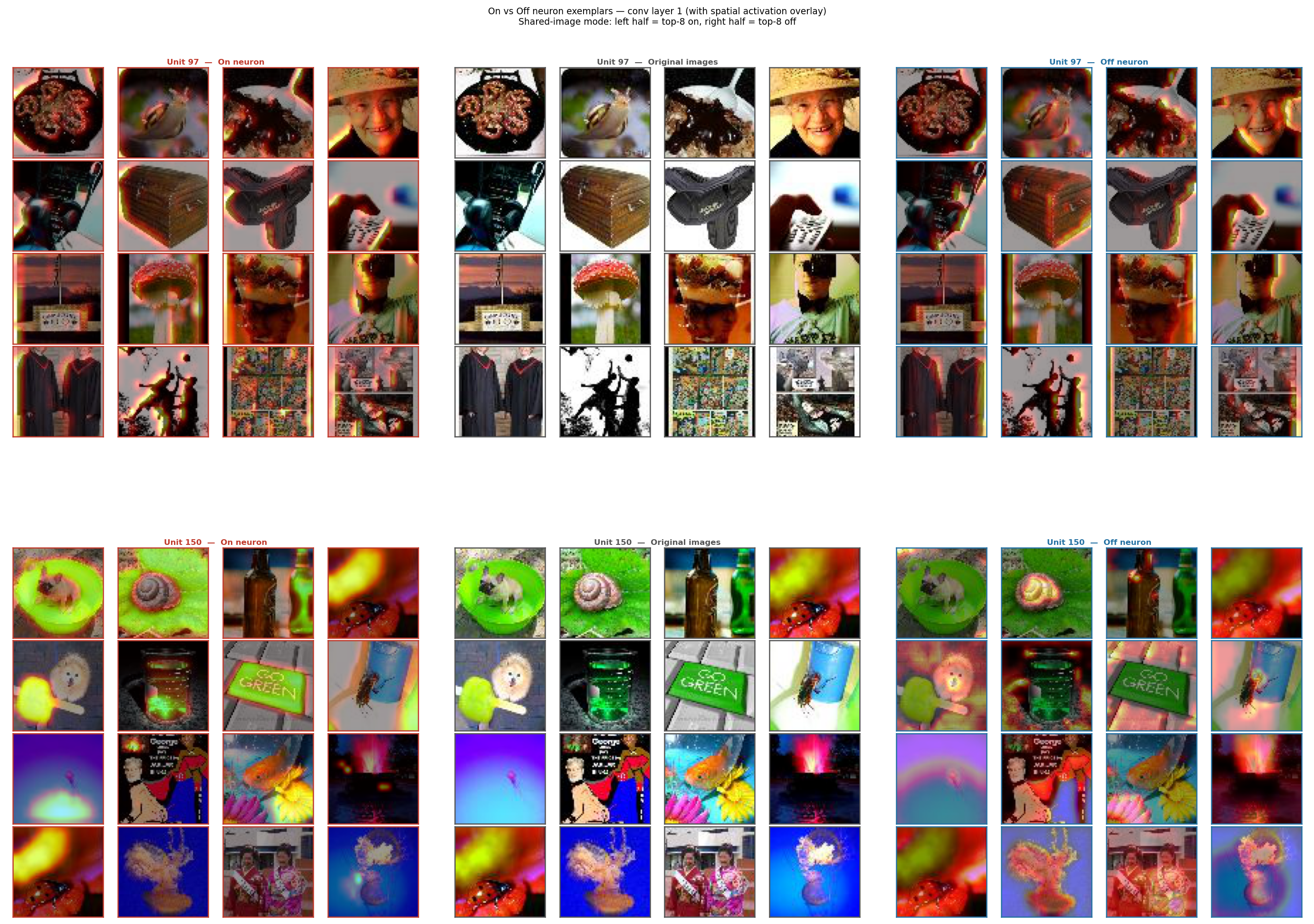}
  \caption{
  \textbf{Learned complementary on-off activations.}
  Activation maps from two representative on-off units sampled from the second convolutional layer of the Tiny ImageNet model. Activation maps are visualized using the activation-based method of \citet{zeiler2014visualizing}. For each unit, images are selected as the union of the top 8 test images ranked by maximum On-channel activation and the top 8 test images ranked by maximum Off-channel activation. For every selected image, the original image is shown together with both corresponding activation maps, allowing comparison between the On and Off responses for the same input. Activation maps are post-$\mathrm{ReLU}^{\pm}$ responses, bilinearly upsampled to the input resolution, and normalized independently per image for visualization.
  \textbf{Top:} A boundary-sensitive unit, where the paired channels respond to complementary edge or boundary patterns.
  \textbf{Bottom:} A color- or contrast-sensitive unit, where the paired channels respond to complementary object and background structure.
  These examples illustrate that paired on-off channels can learn complementary visual representations rather than redundant duplicated features.
  }
  \label{fig:on_off_activation_visualization}
\end{figure}

\subsection{Full proof of gradient recovery}
\label{app:gradient_recovery}

In this section, we show that the proposed on-off model can recover the backpropagation weight update while using only non-negative neuronal activities and sign-constrained synapses. The key idea is that both activations and error signals are represented by pairs of non-negative channels, whose difference corresponds to the effective signed quantity. The proof is stated for a fully connected network with ReLU activations, but the argument extends directly to other architectures whose layers can be written as linear maps followed by element-wise nonlinearities.

\paragraph{Notation}

For any signed vector $v$, define its positive and negative parts by
\begin{equation*}
    v^{+} = \max(v,0), 
    \qquad 
    v^{-} = \max(-v,0),
\end{equation*}
so that
\begin{equation*}
    v = v^{+} - v^{-},
    \qquad 
    v^{+},v^{-} \geq 0.
\end{equation*}
The on-off architecture represents signed quantities using such non-negative channel pairs.

We consider a feedforward network with effective signed weights $W_l$, pre-activations $z_l$, and activations $h_l$:
\begin{equation*}
    z_l = W_l h_{l-1},
    \qquad
    h_l = \sigma(z_l),
\end{equation*}
where $\sigma$ is ReLU, and $W_l$ is a sign-constrained synaptic weight matrix.

\paragraph{Backpropagation} 

Backpropagation computes signed error signals recursively. Let $\epsilon_l$ denote the conventional backpropagation error at layer $l$, corresponding to the derivative of the loss with respect to the pre-activation $z_l$:
\begin{equation*}
    \epsilon_l = \frac{\partial L}{\partial z_l}.
\end{equation*}
Then
\begin{equation*}
    \epsilon_{l-1}
    =
    D_{l-1} W_l^T \epsilon_l,
    \label{eq:app_bp_recursion}
\end{equation*}
where
\begin{equation*}
    D_{l-1} = \mathrm{diag}\big(\sigma'(z_{l-1})\big)
\end{equation*}
is the ReLU derivative gate. The gradient with respect to the weights is
\begin{equation}
    \frac{\partial L}{\partial W_l}
    =
    \epsilon_l h_{l-1}^T.
    \label{eq:app_bp_gradient}
\end{equation}

For gradient descent, the desired update is therefore
\begin{equation*}
    \Delta W_l
    =
    -\eta \epsilon_l h_{l-1}^T.
    \label{eq:app_gd_update}
\end{equation*}

Equivalently, define the descent signal
\begin{equation*}
    \delta_l = -\epsilon_l.
\end{equation*}
Then the desired update can be written as
\begin{equation*}
    \Delta W_l = \eta \delta_l h_{l-1}^T.
    \label{eq:app_descent_update}
\end{equation*}

\paragraph{Non-negative representation of error signals}

The on-off model represents each signed descent signal $\delta_l$ using two non-negative top-down channels:
\begin{equation*}
    \delta_l
    =
    \bar{\partial}_l - \underline{\partial}_l,
    \qquad
    \bar{\partial}_l,\underline{\partial}_l \geq 0.
    \label{eq:app_delta_decomp}
\end{equation*}
Here, $\bar{\partial}_l$ and $\underline{\partial}_l$ denote the positive and negative top-down channels, respectively. Importantly, neither channel contains negative neuronal activity; the signed signal exists only as their effective difference.

At the output layer, the channels are initialized from the descent signal:
\begin{equation*}
    \bar{\partial}_L = \max(\delta_L,0),
    \qquad
    \underline{\partial}_L = \max(-\delta_L,0),
\end{equation*}
which ensures that
\begin{equation*}
    \bar{\partial}_L - \underline{\partial}_L = \delta_L.
\end{equation*}

\paragraph{Top-down propagation through on-off channels}

We now show that if the top-down pathway uses the corresponding on-off structure, then the effective difference between the two non-negative top-down channels follows the same recursion as backpropagation.

While the forward weight matrix is $W_l$, the top-down pathway uses different sets of weights: $W_l^+$ for the channel that contributes positive updates, and $W_l^-$ for the second top-down channel. The positive top-down channel receives contributions from positive-to-positive routes, while the negative top-down channel receives contributions from negative-to-negative routes:
\begin{align}
    \bar{\partial}_{l-1}
    &=
    D_{l-1}
    \left(
        W_l^{+T} \bar{\partial}_l
    \right),
    \label{eq:app_td_pos}
    \\
    \underline{\partial}_{l-1}
    &=
    D_{l-1}
    \left(
        W_l^{-T} \underline{\partial}_l
    \right).
    \label{eq:app_td_neg}
\end{align}
All terms in these equations are non-negative. The matrix $D_{l-1}$ implements the gating induced by the corresponding bottom-up activity. For ReLU, this gate is active exactly when the corresponding bottom-up unit is active, and therefore plays the same role as the ReLU derivative in backpropagation.

Taking the difference between the two channels gives:
\begin{align*}
    \bar{\partial}_{l-1} - \underline{\partial}_{l-1}
    &=
    D_{l-1}
    \left(
       W_l^{+T} \bar{\partial}_l
        -
        W_l^{-T} \underline{\partial}_l
    \right)
\end{align*}
Assuming a symmetric connectivity, where $W_l = W_l^{+T} = W_l^{-T}$, we get:
\begin{align*}
    \bar{\partial}_{l-1} - \underline{\partial}_{l-1}
    &=
    D_{l-1}
    \left(
       W_l^{+T} \bar{\partial}_l
        -
        W_l^{-T} \underline{\partial}_l
    \right)
    \\
    &=
    D_{l-1}
    W_l
    \left(
        \bar{\partial}_l - \underline{\partial}_l
    \right)
    \\
    &=
    D_{l-1} W_l^T \delta_l.
\end{align*}

Using the definition
\begin{equation*}
    \delta_{l-1}
    =
    \bar{\partial}_{l-1} - \underline{\partial}_{l-1},
\end{equation*}
we obtain
\begin{equation}
    \delta_{l-1}
    =
    D_{l-1} W_l^T \delta_l.
    \label{eq:app_our_recursion}
\end{equation}

Equation~\eqref{eq:app_our_recursion} is exactly the backpropagation recursion for the descent signal. Thus, although every top-down channel is non-negative, their difference propagates the same signed information as backpropagation.

\paragraph{Equivalence of the local update to gradient descent}

We now show that the local learning rule recovers the gradient descent update. The Hebbian update for an effective synapse uses the presynaptic bottom-up activity and the two corresponding top-down postsynaptic signals, from the positive and negative channels:

\begin{equation*}
    \Delta W_l
    =
    \eta \bar{\partial}_l h_{l-1}^T -\eta  \underline{\partial}_l
    h_{l-1}^T    \\ 
    =
    \eta
    \left(
        \bar{\partial}_l - \underline{\partial}_l
    \right)
    h_{l-1}^T.
\end{equation*}
Since the on-off channels represent
\begin{equation*}
    \delta_l = \bar{\partial}_l - \underline{\partial}_l,
\end{equation*}
this can be written as

\begin{equation*}
    \Delta W_l
    =
    \eta
    \delta_l
    h_{l-1}^T.
\end{equation*}

Substituting $\delta_l = -\epsilon_l$, where $\epsilon_l = \partial L / \partial z_l$, gives
\begin{equation*}
    \Delta W_l
    =
    -\eta
    \epsilon_l h_{l-1}^T.
\end{equation*}
Using Equation~\eqref{eq:app_bp_gradient}, we obtain
\begin{equation*}
    \Delta W_l
    =
    -\eta
    \frac{\partial L}{\partial W_l}.
\end{equation*}
Therefore, the local Hebbian update recovers the gradient descent update.

\paragraph{Theorem}

\begin{theorem}
Consider a feedforward ReLU network implemented using on-off units, with weights $W_l$. Assume that the top-down pathway uses the transposed on-off channel structure in Equations~\eqref{eq:app_td_pos}--\eqref{eq:app_td_neg}, and that the bottom-up activity gates the top-down propagation according to the ReLU derivative. Then the effective top-down signal
\begin{equation}
    \delta_l = \bar{\partial}_l - \underline{\partial}_l
\end{equation}
satisfies the same recursion as backpropagation. Consequently, the local Hebbian update recovers the gradient descent update for the effective weights.
\end{theorem}

\begin{proof}
The output-layer initialization ensures that the on-off top-down channels represent the correct descent signal:
\begin{equation*}
    \bar{\partial}_L - \underline{\partial}_L = \delta_L.
\end{equation*}
Assume that for some layer $l$,
\begin{equation}
    \bar{\partial}_l - \underline{\partial}_l = \delta_l.
\end{equation}
By the on-off top-down propagation rule,
\begin{align*}
    \bar{\partial}_{l-1}
    &=
    D_{l-1}
    \left(
        W_l^T \bar{\partial}_l
    \right), \\
    \underline{\partial}_{l-1}
    &=
    D_{l-1}
    \left(
        W_l^T \underline{\partial}_l
    \right).
\end{align*}
Taking the difference yields
\begin{align*}
    \bar{\partial}_{l-1} - \underline{\partial}_{l-1}
    &=
    D_{l-1}
    W_l^T
    (\bar{\partial}_l - \underline{\partial}_l) \\
    &=
    D_{l-1} W_l^T \delta_l.
\end{align*}
This is exactly the recursion for the descent signal corresponding to backpropagation. By induction, the equality holds for all layers.

Finally, the local update rule gives
\begin{equation}
    \Delta W_l
    =
    \eta
    \left(
        \bar{\partial}_l - \underline{\partial}_l
    \right)
    h_{l-1}^T
    =
    \eta \delta_l h_{l-1}^T
    =
    -\eta
    \frac{\partial L}{\partial W_l}.
\end{equation}
Thus the local update recovers the gradient descent update.
\end{proof}

\paragraph{Separate bottom-up and top-down weights}

The equivalence above assumes that the connectivity of the two top-down pathways is the transposed on-off structure corresponding to the bottom-up weights. In the model, bottom-up and top-down weights are parameterized separately rather than explicitly tied. However, the Hebbian rule applies identical updates to the corresponding bottom-up and top-down synapses. Therefore, if the pathways are initialized identically, they remain identical throughout training. If they are initialized close to alignment, for example by initializing close to zero, their difference remains small under matched updates, up to noise or other perturbations.

This explains why the symmetric setting recovers the gradient update exactly, while the weakly symmetric setting approximates it. In contrast, when bottom-up and top-down weights are initialized far from alignment, the propagated top-down signals no longer precisely match backpropagation errors, and learning becomes a more approximate form of gradient-based optimization.


\newpage
\section*{NeurIPS Paper Checklist}

\begin{enumerate}

\item {\bf Claims}
    \item[] Question: Do the main claims made in the abstract and introduction accurately reflect the paper's contributions and scope?
    \item[] Answer: \answerYes{} 
    \item[] Justification: the main claims made in the abstract and introduction accurately reflect the paper's contributions and scope

\item {\bf Limitations}
    \item[] Question: Does the paper discuss the limitations of the work performed by the authors?
    \item[] Answer: \answerYes{} 
    \item[] Justification: End of the paper

\item {\bf Theory assumptions and proofs}
    \item[] Question: For each theoretical result, does the paper provide the full set of assumptions and a complete (and correct) proof?
    \item[] Answer: \answerYes{} 
    \item[] Justification: Appendix A.4

    \item {\bf Experimental result reproducibility}
    \item[] Question: Does the paper fully disclose all the information needed to reproduce the main experimental results of the paper to the extent that it affects the main claims and/or conclusions of the paper (regardless of whether the code and data are provided or not)?
    \item[] Answer: \answerYes{} 
    \item[] Justification: Main text + Appendix A.2

\item {\bf Open access to data and code}
    \item[] Question: Does the paper provide open access to the data and code, with sufficient instructions to faithfully reproduce the main experimental results, as described in supplemental material?
    \item[] Answer: \answerYes{} 
    \item[] Justification: Code is provided in the supplemental material

\item {\bf Experimental setting/details}
    \item[] Question: Does the paper specify all the training and test details (e.g., data splits, hyperparameters, how they were chosen, type of optimizer) necessary to understand the results?
    \item[] Answer: \answerYes{} 
    \item[] Justification: Text + code

\item {\bf Experiment statistical significance}
    \item[] Question: Does the paper report error bars suitably and correctly defined or other appropriate information about the statistical significance of the experiments?
    \item[] Answer: \answerYes{} 
    \item[] Justification: bars presented in the tables

\item {\bf Experiments compute resources}
    \item[] Question: For each experiment, does the paper provide sufficient information on the computer resources (type of compute workers, memory, time of execution) needed to reproduce the experiments?
    \item[] Answer: \answerYes{} 
    \item[] Justification: Appendix A.2
    
\item {\bf Code of ethics}
    \item[] Question: Does the research conducted in the paper conform, in every respect, with the NeurIPS Code of Ethics \url{https://neurips.cc/public/EthicsGuidelines}?
    \item[] Answer: \answerYes{} 
    \item[] Justification: the research conducted in the paper conform with the NeurIPS Code of Ethics

\item {\bf Broader impacts}
    \item[] Question: Does the paper discuss both potential positive societal impacts and negative societal impacts of the work performed?
    \item[] Answer: \answerNA{} 
    \item[] Justification: Implications to neuroscience
    
\item {\bf Safeguards}
    \item[] Question: Does the paper describe safeguards that have been put in place for responsible release of data or models that have a high risk for misuse (e.g., pre-trained language models, image generators, or scraped datasets)?
    \item[] Answer: \answerNA{} 
    \item[] Justification: No risks

\item {\bf Licenses for existing assets}
    \item[] Question: Are the creators or original owners of assets (e.g., code, data, models), used in the paper, properly credited and are the license and terms of use explicitly mentioned and properly respected?
    \item[] Answer: \answerYes{} 
    \item[] Justification: References are provided

\item {\bf New assets}
    \item[] Question: Are new assets introduced in the paper well documented and is the documentation provided alongside the assets?
    \item[] Answer: \answerYes{} 
    \item[] Justification: A detailed git page will be published

\item {\bf Crowdsourcing and research with human subjects}
    \item[] Question: For crowdsourcing experiments and research with human subjects, does the paper include the full text of instructions given to participants and screenshots, if applicable, as well as details about compensation (if any)? 
    \item[] Answer: \answerNA{} 
    \item[] Justification:  the paper does not involve crowdsourcing nor research with human subjects.

\item {\bf Institutional review board (IRB) approvals or equivalent for research with human subjects}
    \item[] Question: Does the paper describe potential risks incurred by study participants, whether such risks were disclosed to the subjects, and whether Institutional Review Board (IRB) approvals (or an equivalent approval/review based on the requirements of your country or institution) were obtained?
    \item[] Answer: \answerNA{} 
    \item[] Justification: the paper does not involve crowdsourcing nor research with human subjects

\item {\bf Declaration of LLM usage}
    \item[] Question: Does the paper describe the usage of LLMs if it is an important, original, or non-standard component of the core methods in this research? Note that if the LLM is used only for writing, editing, or formatting purposes and does \emph{not} impact the core methodology, scientific rigor, or originality of the research, declaration is not required.
    \item[] Answer: \answerNA{} 
    \item[] Justification: the core method development in this research does not involve LLMs

\end{enumerate}

\end{document}